\documentclass[letterpaper, 10 pt, conference]{ieeeconf}
\IEEEoverridecommandlockouts                              
\usepackage[T1]{fontenc}

\usepackage{cite}
\usepackage{color}
\let\labelindent\relax
\usepackage{amsmath,amssymb,amsfonts, enumitem}
\usepackage{bbm, dsfont}
\usepackage{algorithm, algpseudocode}
\usepackage{graphicx}
\usepackage{textcomp}
\usepackage{mathtools}
\usepackage{subcaption}

\newtheorem{lemma}{Lemma}
\newtheorem{theorem}{Theorem}

\newtheorem{assumption}{Assumption}
\newtheorem{remark}{Remark}

\newtheorem{prop}{Proposition}

\usepackage[textsize=tiny]{todonotes}
 \usepackage{makecell,booktabs}

\usepackage{setspace}

\DeclareMathAlphabet{\pazocal}{OMS}{zplm}{m}{n}

\newcommand\scalemath[2]{\scalebox{#1}{\mbox{\ensuremath{\displaystyle #2}}}}

\renewcommand{\S}{\pazocal{S}}
\newcommand{\A}{\pazocal{A}}

\newcommand{\R}{\pazocal{R}}
\newcommand{\M}{\pazocal{M}}
\newcommand{\bP}{\mathbb{P}}
\newcommand{\bR}{\mathbb{R}}
\newcommand{\Bstar}{B^\star}
\newcommand{\Wstar}{W^\star}
\newcommand{\wstar}{w^\star}
\newcommand{\hB}{\widehat{B}}

\newcommand{\thetas}{\theta^{\star}}
\newcommand{\Thetas}{\Theta^{\star}}
\newcommand{\wTheta}{\widehat{\Theta}}
\newcommand{\sigmin}{\sigma_{\min}}
\newcommand{\sigmax}{\sigma_{\max}}

\newcommand{\SD}{{\text{SD}}}

\newcommand{\nn}{\nonumber}

\begin{document}
\title{Provable Multi-Task Reinforcement Learning: A Representation Learning Framework with Low Rank Rewards
\thanks{Y. Guo and S. Moothedath are with Electrical and Computer Engineering, Iowa State University. Email: \{yaozeguo, mshana\}@iastate.edu. }
\thanks{This work was supported by the U.S. National
Science Foundation under Grant NSF CAREER 2440455.}}
\author{Yaoze Guo and Shana Moothedath~\IEEEmembership{Senior Member,~IEEE}}

\maketitle

%%%%%%%%%%%%%%%%%%%%%%%%%%%%%%%%%%%%%%%%%%%%%%%%%%%%%%%%%%%%%%%%%%%%%%%%%%%%%%%%
\begin{abstract}
Multi-task representation learning (MTRL) is an approach that learns shared latent representations across related tasks, facilitating collaborative learning that improves the overall learning efficiency. This paper studies MTRL for multi-task reinforcement learning (RL), where multiple tasks have the same
state-action space and transition probabilities, but different rewards.
 We consider $T$ linear Markov Decision Processes (MDPs) where the reward functions and transition dynamics admit linear feature embeddings of dimension $d$. The relatedness among the tasks is captured by a low-rank structure on the reward matrices. 
Learning shared representations across multiple RL tasks is challenging due to the complex and policy-dependent nature of data that leads to a temporal progression of error. 
% {\cblue We consider $T$ MDP tasks that share the same state, action spaces, and probability dynamics, however, different reward functions. We consider a linear MDP setting where the reward functions and transition dynamics admit linear feature embeddings of dimension $d$. The relatedness among the tasks are captured by a low-rank structure.}
Our approach adopts a reward‑free reinforcement learning framework to first learn a data‑collection policy. This policy then informs an exploration strategy for estimating the unknown reward matrices. Importantly, the data collected under this well-designed policy enable accurate estimation, which ultimately supports the learning of an $\epsilon$-optimal policy.
%Our approach combines spectral initialization with task-wise least squares estimation to jointly learn a shared representation and task-specific parameters. 
Unlike existing approaches that rely on restrictive assumptions such as Gaussian features, incoherence conditions, or access to optimal solutions, we propose a low-rank matrix estimation method that operates under more general feature distributions encountered in RL settings. 
Theoretical analysis establishes that accurate low-rank matrix recovery is achievable under these relaxed assumptions, and we characterize the relationship between representation error and sample complexity. Leveraging the learned representation, we construct near-optimal policies and prove a regret bound. 
Experimental results demonstrate that our method effectively learns robust shared representations and task dynamics from finite data.
\end{abstract}
%%%%%%%%%%%%%%%%%%%%%%%%%%%%%%%%%%%%%%%%%%%%%%%%%%%%%%%%%%%%%%%%%%%%%%%%%%%%%%%%
\section{Introduction}
Markov Decision Processes (MDPs) are a widely used mathematical framework for modeling sequential decision-making in dynamical systems and for learning optimal control policies via reinforcement learning (RL) \cite{sutton1998reinforcement,szepesvari2022algorithms}. In many control applications, such as robotics, autonomous systems, and industrial automation, multiple agents interact among themselves and the environment to accomplish complex, coordinated tasks. 
These tasks are often related yet distinct, sharing underlying structure while differing in specifics.
Consider a fleet of autonomous robots in a shared environment, where each robot or mission defines a task. The system dynamics are shared, and actions impact future states. While all tasks aim for common objectives, energy efficiency, safety, and task completion, they weight them differently, resulting in distinct, however related, reward functions \cite{caruana1997multitask}.
{\em We ask: Should tasks be learned independently, or can we exploit the shared structure through joint learning, leading to improved efficiency and generalization?}

Multi-task  representation learning provides a principled approach to answer this question affirmatively by learning shared representations that capture the common structure among tasks \cite{du2020few, tripuraneni2021provable}. In this paper, we study multi-task representation learning for linear MDPs, where the reward matrix for the $T$ tasks/users shares a latent low-dimensional structure. We propose a method that jointly learns the shared representation and task-specific components, and we provide theoretical guarantees on its performance, highlighting the relationship between representation error and regret. 

Unlike in supervised or bandit settings, where feature vectors can be freely designed or assumed to be well-conditioned, in RL settings, data are collected through sequential interaction with the environment, and feature distributions highly depend on the underlying policy and system dynamics. Most work in low rank matrix recovery and representation learning relies on samples from standard Gaussian distributions \cite{nayer2022fast,tripuraneni2021provable,linprovably}, or assumes access to the optimal solution \cite{nagaraj2023multi,hu2021near,yang2020impact, du2020few}. In contrast, we consider a more general and realistic setting in which the sample distribution may be different from these idealized assumptions. These relaxations make our results applicable to more practical scenarios, for example, in reinforcement learning environments where idealized conditions rarely hold.

Our main contributions are as follows.
\begin{itemize}
    \item We propose a multi-task representation learning framework for reinforcement learning (MTRL-RL) that collaboratively explores rewards across $T$ task-specific MDPs and learns the reward matrices via a low-rank estimation approach with convergence guarantees.
    \item We establish theoretical guarantees showing that low rank matrix recovery remains possible despite the absence of standard assumptions, such as incoherence and Gaussian features, which demonstrate that low rank matrix estimation can still succeed in realistic reinforcement learning environments. 
    \item We combine the learned representation with the system dynamics to construct $\epsilon$-optimal policies and establish a regret bound of $\mathcal{O}(NTH\sqrt{d}\delta_0)$, where $H$ denotes the horizon length and $N$ the number of rounds.
    \item We empirically validated the proposed approach on two MDP environments. Our analysis demonstrates that our proposed algorithm accurately recovers the reward matrices and substantially outperforms baseline methods in both estimation accuracy and regret.
    %Applying these learned parameters, we can construct a $\epsilon$-optimal policy.
\end{itemize}
\noindent{\bf Organization:} The remainder of the paper is organized as follows.
In Section~\ref{sec:rel}, we present the related work.
In Section~\ref{sec:prob}, we present the problem setting. In Section~\ref{sec: Proposed Algorithm}, we present the proposed algorithm. In Section~\ref{sec:analysis}, we present the theoretical analysis and convergence results. In Section~\ref{sec: Simulations}, we present the simulation results. In Section~\ref{sec:Conclusion}, we present the conclusions and potential directions for future work. 
%
% \begin{table*}[h]
% \caption{}
% \label{table: complexities}
% \centering
% \resizebox{0.99\linewidth}{!}{
% \begin{tabular}{|l|c|c|c|c|c|}
% \toprule
% \textbf{Method} 
% & \textbf{Sample Complexity}
% & \textbf{Regret}
% & \textbf{Assumptions} 
% & \textbf{Column-wise error bound}\\
% \midrule
%
% \textbf{Nayer et al}
% & $dr^2\log(\frac{1}{\delta_0})$
% & N/A
% & Incoherence, Standard Gaussian 
% & $\delta_0\|x^*\|$\\
% \midrule
%
% \textbf{Lin et al }
% & $\frac{dr^2}{\delta_0^2}$
% & $\sigma_{\max}^\star\sqrt{rNT}\delta_0$
% & Incoherence, Standard Gaussian
% & $\sqrt{\frac{r}{T}}\delta_0\sigma^*_{\text{max}}$\\
% \midrule
%
% \textbf{Nagaraj et al}
% & $r(T+d)\log T$
% & N/A
% & Optimal solution is available
% & $0$ \\
% \midrule
 %
% \textbf{Proposed*}
% & $\frac{r(T+d)d^2}{\delta_0^2}$
% & $NTH\sqrt{d}\delta_0$
% & Proposition~\ref{prop:cond4}
% & $\delta_0\sqrt{d}$\\
% \bottomrule
% \end{tabular}
% }
% \end{table*}
%
\section{Related Work}\label{sec:rel}

\noindent{\bf Multi-task reinforcement learning:}
Multi‑task reinforcement learning considers a setting in which an agent jointly learns to solve multiple related reinforcement‑learning tasks.  The agent encounters a collection of MDPs that share common structure, such as operating within the same environment while pursuing different objectives, or performing similar objectives. This paradigm has been studied extensively from both empirical \cite{hessel2019multi,sodhani2021multi} and theoretical perspectives \cite{brunskill2013sample, teh2017distral}. 
%
% Multi-task reinforcement learning has been widely studied. A variety of methods have been proposed to quantify task similarity in multi-task reinforcement learning. \cite{teh2017distral} assumed policies share common components to capture common behaviors. Actor-critic methods have been extended to multi-task settings by sharing lower-layer neural network parameters while keeping task dependent output heads \cite{zhang2021multi,reverdy2016satisficing}. Additionally, recent studies investigate experience sharing and attention mechanisms within actor-critic frameworks in multi-task/agent settings \cite{christianos2020shared,iqbal2019actor}. 
The success of these approaches highlights the importance of learning shared knowledge and task-specific adaptation in multi-task reinforcement learning.

% \cite{borsa2016learning} \cite{ishfaq2024offline} and \cite{cheng2022provable} assume common features for transition dynamics among different tasks. 

% \cite{nagaraj2023multi} and \cite{bai2025accelerating} assume the reward matrix has rank $r < \min\{|S|,N\}$, where $N$ is the number of tasks. \cite{bai2025accelerating} applied Singular Value Decomposition (SVD) technique in TD learning while \cite{nagaraj2023multi} applied Low Rank Matrix Completion (LRMC) to estimate the reward matrix. 
\noindent{\bf Low-rank matrix recovery:} Low-rank matrix estimation has been extensively studied in statistics and machine learning. Classic approaches include nuclear norm minimization \cite{candes2012exact,wainwright2019high} and spectral methods \cite{linprovably,nayer2022fast,tripuraneni2021provable}, which achieve strong guarantees under assumptions such as incoherence and random samples are drawn from a standard Gaussian distribution. However, such assumptions are often unrealistic in many real-world scenarios, such as RL environments, where feature vectors are inherently policy-dependent. Our work analyzes low rank matrix recovery under more general feature distributions that arise in RL settings.

\noindent{\bf Low-rank reinforcement learning:} Low-rank reinforcement learning has gathered significant recent interest.  
Existing work typically imposes a low-rank structure on either the transition dynamics or the reward function. 
% \cite{nagaraj2023multi} assume the reward matrix has low rank in their tabular setting. 
% and \cite{bai2025accelerating} assume the reward matrix has low rank. 
\cite{cheng2022provable} and \cite{ishfaq2024offline} assume a low rank structure in transition dynamics, while \cite{bai2025accelerating} assumes a low rank structure in the value functions across all tasks. \cite{jin2023provably} assumes that both the transition dynamics and reward function are linear. 
Our work builds upon the multi-user RL framework with low-rank rewards introduced by \cite{nagaraj2023multi}, which serves as a foundational basis for our approach.
\section{Problem Formulation}\label{sec:prob}
\subsection{Notations}
We use $\|\cdot\|$ to denote the Euclidean norm, $\|\cdot\|_1$ to denote the $\ell_1$ norm, $\|\cdot\|_\text{F}$ to denote the Frobenius norm, $\|\cdot\|_{\psi_2}$ to denote the sub-Gaussian norm, and $\|\cdot\|_{\psi_1}$ to denote the sub-exponential norm. We use $e_1,...,e_m$ to denote the standard basis vectors of $\mathbb{R}^m$, for some $m \in \mathbb{N}$. Let $S^{d-1} := \{x \in \mathbb{R}^d:||x||=1\}$ and $\mathcal{B}_d(r):= \{x \in \mathbb{R}^d:||x||\leqslant r\}$. $c$ and $C$ are positive constants with $c < 1$ and $C > 1$ respectively. The subspace distance between two orthonormal matrices $B_1,B_2$ is denoted as $\text{SD}(B_1, B_2) = \|(I-B_1B_1^\top)B_2\|$. 

\subsection{Multi-Task Reinforcement Learning}
We consider a multi-task reinforcement learning problem. Each task $t \in [T]$ is associated with a finite horizon episodic MDP that shares the same state space $\S$, action space $\A$, horizon $H$, and transition kernel $\{\mathbb{P}_h\}_{h=1}^H$, where $\mathbb{P}_h(\cdot|s,a)$ gives the distribution of the state in step $h+1$ if action $a\in\A$ is taken for state $s\in\S$ in step $h$. Each task has a different deterministic reward denoted by $R_t = (R_{1t},...,R_{Ht})$, where $R_{ht}:\S\times \A \rightarrow[0,1]$ denotes the reward for task $t$ in step $h$ for selecting an action $a \in \A$ in state $s\in \S$.  We denote the MDP associated with task $t\in [T]$ by a tuple $\M_t=(\S, \A, \{R_{ht}\}_{h=1}^H, \{\bP_h\}_{h=1}^H)$. 

We consider all episodic MDPs start at a random state $s_1$ with the same distribution for all $t \in [T]$. 
A policy $\Pi:=(\pi_1, \ldots, \pi_H)$, where $\pi_h: \mathcal{S} \to \Delta(\A)$,  is a mapping from states to actions.
Then, at each $h\in [H]$, we observe a state $s_h \in \S$, choose an action $a_h \in \A$ according to its policy $\pi_h$ and receive a reward $R_{ht}(s_h,a_h)$. Further, the MDP evolves into a new state $s_{h+1}\sim \bP_h(\cdot|s_h, a_h)$. Note that we cannot take action at step $H+1$, hence receive no rewards. 

The value function $V_t^\Pi: \mathcal{S} \to \mathbb{R}$ is defined as the expected cumulative reward obtained by following policy $\Pi$ starting from the state $s$ at step $h$ for task $t$. Formally, for all $s_h \in \mathcal{S},\; h \in [H]$, we have
\[
V_t^\Pi(s_h) = \mathbb{E}_\Pi\left[\sum_{h'=h}^{H} R_{h't}(s_{h'}, \pi_{h'}(s_{h'})) \mid s_h \right].
\]
Similarly, the action-value function $Q_t^\Pi: \mathcal{S} \times \mathcal{A} \to \mathbb{R}$ represents the expected cumulative reward when the agent starts from state $s_h$, takes action $a_h$ at step $h$ for task $t$, and then follows policy $\Pi$ afterward. For all $(s_h,a_h) \in \mathcal{S} \times \mathcal{A},\; h \in [H],$
\begin{align*}
    Q_t^\Pi(s_h,a_h) &= R_{ht}(s_h,a_h) \\&+ \mathbb{E}_\Pi\left[\sum_{h'=h+1}^{H} R_{h't}(s_{h'}, \pi_{h'}(s_{h'})) \mid s_h, a_h \right],
\end{align*}
Let $V^\star_t(s_h^n)$ be the optimal value function of task $t$ for a state $s$ at horizon $h$ and episode $n$ under optimal policy $\Pi^\star(t)$. $\Pi\epsilon(t)$ is called an $\epsilon$-optimal policy for task $t$ if $V^{\Pi\epsilon(t)}_t(s_1) \geq V^\star_t(s_1) - \epsilon$. Let $\hat{\Pi}^{\star}(t)$ be the optimal policy for task $t$ induced by the estimated rewards. 
For simplicity, we assume the state space and action space are finite sets, and rewards are deterministic. 
\subsection{Linear MDP}
We consider a linear MDP setting where the reward functions and transition dynamics admit linear feature embeddings. We consider embedding feature $\psi : \S \times \A \rightarrow \mathbb{R}^d$ and $\phi : \S \times \A \rightarrow \mathbb{R}^d$ such that
\begin{enumerate}
    \item $||\psi(s,a)|| \leqslant 1$, $||\phi(s,a)||_1 \leqslant 1$
    \item There exist reward parameters $\thetas_{ht} \in \mathbb{R}^d$, $||\thetas_{ht}|| \leqslant \sqrt{d}$ such that $R_{ht}(s,a) = \langle\thetas_{ht},\psi(s,a)\rangle$ and $R_{ht}(s,a) \in [0,1]$.
    \item There exist signed measures $\mu_h^1, \ldots, \mu_h^d$ over the space $\mathcal{S}$ such that
\(
\bP_h(\cdot \mid s,a) = \sum_{i=1}^{d} \mu_h^i(\cdot)\,\langle \phi(s,a), e_i \rangle.
\)
\end{enumerate}
We denote the $T \times d$ matrix whose $t$-th row is $\thetas_{ht}{}^\top$ to be $\Thetas_h$. We assume a low-rank structure on $\Thetas_h$.
\begin{assumption}
\label{low rank}
    The $T \times d$ matrix $\Thetas_h$ has rank $r \leqslant \frac{1}{2}\min(T,d)$. 
\end{assumption}
Let \(\Thetas_h{}^\top \overset{\text{SVD}}{=} \Bstar_h\Sigma_h^{\star}D_h^{\star} := B_h^{\star}(\Sigma_h^{\star}D^{\star}_h) := B_h^{\star}W_h^{\star}\) denote the reduced (rank $r$) Singular Value Decomposition (SVD) of $\Thetas_h{}^\top$, where $\Wstar_h := \Sigma_h^{\star}D^{\star}_h$. Denote $t$-th column of $\Wstar_h$ as $\wstar_{ht}$. Then we have $\thetas_{ht} = \Bstar_h\wstar_{ht}$, for $h\in [H]$ and $t\in [T]$. Let $\sigmin^\star(h)$ and $\sigmax^\star(h)$ be the minimum and maximum singular value of $\Thetas_h$ for all $h\in[H]$, respectively.
\subsection{Learning Objective}
The goal is to estimate the reward parameter $\thetas_{ht}$ and learn an $\epsilon$-optimal policy $\hat{\Pi}^{\star}(t)$, for all task $t \in [T]$ and $h \in [H]$. 
% Since we only consider a finite horizon problem, $V^\pi_t(s_{H+1}^n)=0,\hat{V}^\pi_t(s_{H+1}^n) = 0, \forall \pi, n \in [N], t\in[T]$.
% Let $\pi^{\star}$ be the optimal policy. Let $\hat{\pi}^{\star}$ be the optimal policy calculated by the estimated matrix $\hat{\Theta}_h$. For simplicity, we assume both $\pi^{\star}$ and $\hat{\pi}^{\star}$ are deterministic policy. Let $a_h^n$ be the action taken by policy $\pi^{\star}$ at horizon $h$ and episode $n$. Let $\hat{a}_h^n$ be the action taken by $\hat{\pi}^{\star}$ at horizon $h$ and episode $n$. 
%
The goal of representation learning is to obtain the estimates
\[
\widehat{\Theta}_h^\top = \widehat{B}_h\widehat{W}_h, \text{~for~all~} h \in [H]
\]
with the goal of minimizing the cost function $f(\widehat{B}_h,\widehat{W}_h)$,
\begin{equation}\label{eq:est}
f(\widehat{B}_h,\widehat{W}_h) =
\sum_{k=1}^{K}\sum_{t=1}^{T}
\left\|y_{tk} - \psi_{tk}^\top \widehat{B}_h \widehat{w}_{ht} \right\|^2, \text{~for~all~} h\in [H],
\end{equation}
where $\widehat{B}_h \in \mathbb{R}^{d \times r}$, $\widehat{W}_h =[\widehat{w}_{h1}, \ldots, \widehat{w}_{hT}] \in \mathbb{R}^{r \times T}$, and $K$ is the number of data samples.
We are also interested in obtaining the regret guarantee over $N$ episodes, given by
\begin{equation}
\text{Reg}(N,T) = \sum_{n=1}^{N} \sum_{t=1}^{T} V^{\star}_t(s_1^n) - V_t^{\hat{\Pi}^{\star}(t)}(s_1^n).\label{eq:reg}
\end{equation}
\section{Multi-Task Representation Learning for Reinforcement Learning (MTRL-RL)}\label{sec: Proposed Algorithm}
%
 %In this section, we propose a multi-task representation learning algorithm for reinforcement learning  (MTRL-RL) to minimize Eqs.~\eqref{eq:reg} and~\eqref{eq:est}. 
 The pseudocode of the MTRL-RL algorithm is presented in Algorithm~\ref{alg:whole}. 
 MTRL-RL proceeds in four stages. In stage~1, we perform a reward free RL where each task explores the shared environment without access to the rewards.
 %The rewards are revealed at the end of the exploration and the RL goal is to learn a near-optimal policy for that reward function.
%Our algorithm proceeds in 4 stages. 
%In the first stage, we perform reward-free reinforcement learning.
Reward‑free reinforcement learning algorithms with strong theoretical performance guarantees can be leveraged in this stage \cite{zhang2020nearly,wagenmaker2022reward}. In reward free RL, we want to explore the MDP such that we can obtain the optimal policy for every possible reward. After collecting sufficient amount of trajectories from MDP sequentially, the algorithm outputs two functions $\check{\Pi}$ and $\check{V}$, where given any reward $\mathcal{R}_t$, $\check{\Pi}(\mathcal{R}_t)$ outputs an optimal policy and $\check{V}(\check{\Pi}(\mathcal{R}_t))$ outputs the optimal value function for this reward. Since all tasks share the same reward free MDP, we select trajectories from random tasks. After completing stage~1, we proceed to stage~2.
%the unknowns will be the rewards for the different tasks. 

Stage 2 consists of two steps that follow the methods of \cite{nagaraj2023multi} to compute an exploration policy. In step 1, obtain $M$ samples $(s_{mh},a_{mh})$ such that $G_{h} := \sum_{m=1}^{M} \phi(s_{mh},a_{mh})\phi(s_{mh},a_{mh})^\top \succeq I$ for all $h \in [H]$. It can be done by Algorithm 2 in \cite{nagaraj2023multi}, which uses $\check{\Pi}(\cdot)$ obtained in stage 1. The outputs of this step are $\{\phi_{mh}:=\phi(s_{mh},a_{mh})\}_{m\in[M],h\in[H]}$, $\{s_{mh}\}_{m\in[M],h\in[H]}$ and $\{G_h\}_{h\in[H]}$. In step 2, the exploration policy $\hat{\Pi} := (\hat{\pi}_1,\cdots,\hat{\pi}_H)$ is constructed by solving the following optimization problem. Let $\xi > 0$ be a constant set manually.
     \begin{align}
         \hat{\pi}_1 &= \arg \sup_{\hat{\pi}_1} \inf_{x\in S^{d-1}} \hat{\mathcal{T}}_1(f,\hat{\pi}_1),
     \label{eqs:pi1}\\
         \hat{\pi}_h &= \arg \sup_{\hat{\pi}_h,\nu\in \mathcal{B}_d(1)} \inf_{x\in S^{d-1}} \hat{\mathcal{T}}_h(f,\nu,\hat{\pi}_h), \text{for } h>1      \label{eqs:pih}
     \end{align} 
     where \[f(s,a,x) := |\langle x,\psi(s,a)\rangle|\sqrt{d} - \xi d\langle x,\psi(s,a)\rangle^2,\] \[\hat{\mathcal{T}}_1(f,\hat{\pi}_1) = \frac{1}{M}\sum_{m=1}^{M}\sum_{ a\in\A}f(s_{m1},a,x)\hat{\pi}_1(a|s_{m1}),\] \[\hat{\mathcal{T}}_h(f,\nu,\hat{\pi}_h) = \frac{1}{M}\sum_{m=1}^{M}\phi_{m(h-1)}^\top G_{h-1}^{-1}\nu\Bigl(\sum_{a\in \A}f(s_{mh},a,x)\hat{\pi}_h(a|s_{mh})\Bigr).\]
%
% In stage 2, we leverage the transition embedding obtained from stage 1 to design collaborative exploration policies $\hat{\Pi}$. Let $S_{1:H},A_{1:H}  \sim  \M(\pi)$ be the random trajectory under policy $\pi$. 
%
% We use $\hat{\Pi}$ in the next stage to collect samples of $\psi(s_h,a_h)$, which is done by running the MDP with policy $\pi$. Let $S_{1:H},A_{1:H}  \sim  \M(\pi)$ be the random trajectory under policy $\pi$.

In stage~3, we perform matrix estimation to recover the low-rank reward matrix $\Thetas_h$ for all $h\in[H]$ using samples collected by running MDPs under policy $\hat{\Pi}$ learned in stage 2 for all tasks. Finally, in stage 4, we use the estimated reward, computed using the reward parameters obtained in stage 3, $\hat{R}_{ht}(s,a) = \langle\hat{\theta}_{ht},\psi(s,a)\rangle$, and the $\check{\Pi}(\cdot)$ learned in stage 1 to compute an optimal policy for each task.
%
% We collect samples of form $(\psi, \langle\theta^*_{ht},\psi\rangle)$ for all $t\in[T]$ and $h\in[T]$ by running MDP under policy $\hat{\Pi}$ constructed in stage 2 for each task $K$ times, where $K$ is the number of samples.

\begin{algorithm}[t]
    \caption{Multi-task Linear Reinforcement Learning}
    \label{alg:whole}
\begin{algorithmic}[1]
     %\State {\bfseries Input:} {\cblue number of samples $K$, $\phi(s,a), \psi(s,a)$, for all $s \in \S, a \in \S$ }
     % \State {\bfseries Output:} $\epsilon$-optimal policy $\hat{\Pi}^\star$
     \Statex \textbf{\textsc{Stage 1: Reward free RL}}
     \State  Randomly sample MDPs and perform reward free reinforcement learning. Solve for $\check{\Pi}(\cdot)$ and $\check{V}(\cdot)$.
     \Statex \textbf{\textsc{Stage 2: Construct exploration policy $\hat{\Pi}$}}
     \State Collect features $\{\phi_{mh}\}_{h\in[H],m\in[M]}$ using $\check{\Pi}(\cdot)$ from stage 1 and construct policy $\hat{\Pi}=(\hat{\pi}_1,\cdots,\hat{\pi}_H)$ by solving Eqs.~\eqref{eqs:pi1} and~\eqref{eqs:pih}.
     % \State \hspace{1em} stage 2 inputs: $\check{\Pi}(\cdot)$, $M$.
     % \State \hspace{1em} stage 2 output: $\hat{\Pi}$.
     % \State Run MDP under policy $\hat{\Pi}$ for each task $K$ times. Observe all $(s,a)$, $\psi(s,a)$ and corresponding reward $y$.
    \Statex \textbf{\textsc{Stage 3: Reward matrix estimation}}
    \For{each $t\in[T]$}
        \For{each $k=1:K$}
            \State \hspace{-1em} Run MDP under policy $\hat{\Pi}$. At each $h\in[H]$, observe $\psi_{tk}(s_h,a_h)$ and $y_{tk}(h) := \langle\thetas_{ht},\psi_{tk}(s_h,a_h)\rangle$.
        \EndFor
    \EndFor
     \For{each $h \in[H]$}
     \State For $t\in[T]$, construct $Y_{ht} = [y_{t1}(h),\ldots, y_{tK}(h)]^\top$, $\Psi_{ht} = [\psi_{t1}(s_h,a_h),\ldots,\psi_{tK}(s_h,a_h)]^\top$, $Y_{ht} \in \bR^K$ and $\Psi_{ht} \in \bR^{K\times d}$
     %\State {\bfseries Spectral Initialization}
     \State $\widehat{\Theta}_0^\top(h):= \frac{1}{K}\sum^{T}_{t=1}\Psi_{ht}^\top Y_{ht}e_t^\top$ 
     \State Set $\hB_h$ to be top $r$ singular vectors of $\wTheta^\top_0(h)$
    \State  Update $\widehat{w}_{ht}$, $\hat{\theta}_{ht}$: For each $t \in [T]$, $\widehat{w}_{ht} = (\Psi_{ht}\hB_h)^\dagger Y_{ht}$, $\hat{\theta}_{ht} = \hB_h\widehat{w}_{ht}$
        % \State Low rank matrix estimation (Algorithm \ref{alg:init})
        \State Set $\widehat{\Theta}_h =[\hat{\theta}_{h1}, \ldots, \hat{\theta}_{hT}]^\top$ for all $h\in [H]$
     \EndFor
      \Statex \textbf{\textsc{Stage 4: Construct $\epsilon$-optimal policy}}
     \State Construct $\hat{\mathcal{R}}_{t}=\{\hat{R}_{ht}\}_{h=1}^H$ for all $t\in[T]$ using $\hat{\theta}_{ht}$, for all $h\in[H]$, and $\psi(s,a)$ for all $s\in\S,a\in\A$
     \State Construct the $\epsilon$-optimal policy $\hat{\Pi}^\star(t) = \check{\Pi}(\hat{\mathcal{R}}_t)$ for each task $t\in [T]$
    \State {\bfseries Output:} $\epsilon$-optimal policy $\hat{\Pi}^\star$ 
\end{algorithmic}
\end{algorithm}
\section{Theoretical Analysis of MTRL-RL}\label{sec:analysis}
\subsection{Preliminaries}
In this subsection, we present the preliminary results used in the analysis of the proposed approach. 
\begin{prop}
\label{prop:wedin}
(Wedin sin$\Theta$ Theorem) \cite{chen2021spectral}. For two matrices $M^\star, M \in \mathbb{R}^{n_1 \times n_2}$, let
$B^\star, B \in \mathbb{R}^{n_1 \times r}$ denote the matrices containing
their top-$r$ left singular vectors, and let
$V^{\star}, V \in \mathbb{R}^{r \times n_2}$ denote the matrices
containing their right singular vectors.
Let $\sigma_r^\star$ and $\sigma_{r+1}^\star$ denote the $r$-th and
$(r+1)$-th singular values of $M^\star$.
If
\(
\| M - M^\star \| \le \sigma_r^\star - \sigma_{r+1}^\star,
\)
then the subspace distance satisfies
\[
\mathrm{SD}(B, B^\star)\le\hspace{-2 mm}\
\frac{
\sqrt{2}
\max\left(
\| (M - M^\star)^\top B^\star \|,
\| (M - M^\star) V^{\star\top} \|
\right)
}{
\sigma_r^\star - \sigma_{r+1}^\star - \| M - M^\star \|
}
\]
Furthermore, if \(\|M-M^{\star}\| \leqslant (1-1/\sqrt{2})(\sigma_r^\star - \sigma_{r+1}^\star)\), then
\[
\mathrm{SD}(B, B^\star)
\;\le\;
\frac{
2\| M - M^\star \|
}{
\sigma_r^\star - \sigma_{r+1}^\star
}.
\]
\end{prop}
\begin{prop}
\label{prop:cond4}[Theorem~6.2, \cite{nagaraj2023multi}]
    Consider the samples $\psi_{tk}$s that are collected following Algorithm~\ref{alg:whole}. Then $\psi_{tk}$ are i.i.d. random vectors such that there exist $\zeta,\xi > 0$ such that for any $x \in S^{d-1}$ we have
    \begin{align}\label{cond_4}
        \|\psi_{tk}\| &\leqslant 1      \text{  almost surely},\nn\\
        \mathbb{E}[|\langle\psi_{tk},x\rangle|] &\geq \frac{\zeta}{\sqrt{d}},\nn\\
        \mathbb{E}[\psi_{tk}\psi_{tk}^\top] &\leqslant \frac{1}{d\xi^2}.
    \end{align}
\end{prop}
% Since the results hold for all $h\in[H]$, for notational simplicity, we drop $h$ from $\widehat{\Theta}_0^\top(h)$ and $\Psi_{th}$. For the remainder of this section, we denote $\widehat{\Theta}_0^\top(h)$ by $\widehat{\Theta}_0$, and $\Psi_{ht}$ by $\Psi_t$.
\begin{prop} \label{prop:subg and others}
Consider matrix estimate $\widehat{\Theta}_0(h)$ for all $h\in[H]$ and random vectors $\psi_{tk}$s from Algorithm~\ref{alg:whole}. Then, 
     \begin{enumerate}        
         \item $\psi_{tk}$ are sub-Gaussian and  $\large\|\langle\psi_{tk},x\rangle\large\|_{\psi_2} \leqslant C,$ for all $\|x\|=1$.
         \item There exists $\zeta,\xi > 0$ such that $$\frac{\zeta^2}{d} \leqslant \mathbb{E}[\psi_{tk}\psi_{tk}^\top] \leqslant \frac{1}{d\xi^2}.$$
    % \begin{align}%\label{assume:cond_4_restated}
    %     \|\psi_{tk}\| &\leqslant 1      \text{  almost surely},\nn\\
    %     \frac{\zeta^2}{d} &\leqslant \mathbb{E}[\psi_{tk}\psi_{tk}^\top] \leqslant \frac{1}{d\xi^2}.\nn
    % \end{align}
        \item For all $h\in[H]$, \(\mathbb{E}[{\wTheta}_0^\top(h)] = \mathbb{E}\big[\psi\psi^\top\big]\Thetas_h{}^\top\), the $r$-th and $(r+1)$-th singular value of $\mathbb{E}[{\wTheta}_0^\top(h)]$ satisfy \(\sigma_r(\mathbb{E}[{\wTheta}_0^\top(h)]) \geq \frac{\zeta^2}{d}\sigmin^*(h)\) and \(\sigma_{r+1}(\mathbb{E}[{\wTheta}_0^\top(h)]) = 0\).
     \end{enumerate}
\end{prop}
\begin{proof}
Our goal is to show that for any unit vector $x$, $\langle\psi_{tk},x\rangle$ is a sub-Gaussian random variable. 
    By the Cauchy-Schwarz inequality and Proposition~\ref{prop:cond4}
    \[
        |\langle\psi_{tk},x\rangle| \leqslant \|\psi_{tk}\|\,\|x\| \leqslant 1 \text{ almost~surely.}
    \]
    Thus there exists a constant $C>0$ such that $|\langle\psi_{tk},x\rangle|$ is a sub-Gaussian random variable with sub-Gaussian norm 
$  \large\|\langle\psi_{tk},x\rangle\large\|_{\psi_2} \leqslant C.$
This proves 1).

To prove 2), it is sufficient to show that $\mathbb{E}[\psi_{tk}\psi_{tk}^\top] \geq \frac{\zeta^2}{d}$. 
For all \(x \in S^{d-1}\), we have \(x^\top\mathbb{E}[\psi_{tk}\psi_{tk}^\top]x = \mathbb{E}[\langle\psi_{tk},x\rangle^2]\geq \mathbb{E}^2[|\langle\psi_{tk},x\rangle|]\geq \frac{\zeta^2}{d}. \) As a result, $\mathbb{E}[\psi_{tk}\psi_{tk}^\top] \geq \frac{\zeta^2}{d}$. The upper bound directly follows from Proposition~\ref{prop:cond4}.

Now we prove 3). We have
\begin{align*}
    \widehat{\Theta}_0^\top(h) &= \frac{1}{K}\sum^{T}_{t=1}\Psi_{ht}^\top Y_{ht} e_t^\top 
    %&= \frac{1}{K}\sum^{T}_{t=1}\Psi_{ht}^\top [\thetas_{ht}{}^\top\psi_{t1},\ldots,\thetas_t{}^\top\psi_{tK}]^\top e_t^\top\\
   %&= \frac{1}{K}\sum^{T}_{t=1}\Psi_{ht}^\top(\thetas_{ht}{}^\top \Psi_{ht}^\top)^\top e_t^\top = \frac{1}{K}\sum^{T}_{t=1}\Psi_{ht}^\top\Psi_{ht} \thetas_{ht} e_t^\top\\
    = \sum^{T}_{t=1}\Big(\frac{1}{K}\sum^{K}_{k=1}\psi_{tk}\psi_{tk}^\top\Big)\thetas_{ht} e_t^\top.
\end{align*}
Therefore,
    \begin{align*}
        \mathbb{E}[\widehat{\Theta}_0^\top(h)] &= \sum^{T}_{t=1}\mathbb{E}\big[\frac{1}{K}\sum^{K}_{k=1}\psi_{tk}\psi_{tk}^\top\big]\thetas_{ht} e_t^\top  \\
          &= \mathbb{E}\big[\psi\psi^\top\big]\sum^{T}_{t=1}\thetas_{ht} e_t^\top=  \mathbb{E}\big[\psi\psi^\top\big]\Thetas_h{}^\top.
    \end{align*}
    Moreover, the $r$-th singular value of $\mathbb{E}[{\wTheta_0^\top(h)}]$ satisfy \(\sigma_r^\star(\mathbb{E}[{\wTheta}_0^\top(h)]) \geq \frac{\zeta^2}{d}\sigmin^\star(h)\) and \(\sigma_{r+1}^\star(\mathbb{E}[{\wTheta}_0^\top(h)]) = 0\) since it is a rank $r$ matrix.
\end{proof}
\subsection{Guarantees on Estimates of the Reward Matrices}
In this subsection, we present the estimation results for the reward matrices for the different horizon lengths.
Since the results hold for all $h\in[H]$, for notational simplicity, we drop $h$ from $\widehat{\Theta}_0^\top(h)$, $\Psi_{th}$ and $\sigmin^\star(h)$. For the remainder of this section, we denote $\widehat{\Theta}_0^\top(h)$ by $\widehat{\Theta}_0$, $\Psi_{ht}$ by $\Psi_t$ and $\sigmin^\star(h)$ by $\sigmin^\star$.
We first establish a bound on the deviation between $\widehat{\Theta}_0$ and its expected value.
\begin{lemma}
\label{theta-Etheta}
    With probability at least $1-\exp[(T+d)-c\epsilon_3^2 \delta_0^2 K\frac{\zeta^4}{d^2}\sigmin^{\star2}]$, the estimate $\widehat{\Theta}_0$ from Algorithm~\ref{alg:whole} satisfies
    \[
        \big\|\widehat{\Theta}_0-\mathbb{E}[\widehat{\Theta}_0]\big\| \leqslant 0.1\delta_0\frac{\zeta^2}{d}\sigma_{\min}^{\star}.
    \]
\end{lemma}
\begin{proof}
    For a fixed $z_1\in S^{d-1}$ and $z_2 \in S^{T-1}$, we have
    \[
        \langle(\widehat{\Theta}_0-\mathbb{E}[\widehat{\Theta}_0])^\top, z_1z_2^\top \rangle = \frac{1}{K}\sum_{tk}z_2(t)\Big(y_{tk}\big(\psi^\top_{tk}z_1\big) - \mathbb{E}[\cdot]\Big),
    \] where $\mathbb{E}[\cdot]$ is the expected value of the first term.
    The summands are independent zero mean sub-Gaussian r.v.s with sub-Gaussian norm $\tau_{tk} \leqslant C|z_2(t)|\cdot|y_{tk}|\cdot\|\psi_{tk}\|/K\leqslant C|z_2(t)|/K$.
    Let $g = \epsilon_3\delta_0\frac{\zeta^2}{d}\sigma_{\min}^{\star}$, we have
    \[
        \frac{g^2}{\sum_{tk}\tau_{tk}^2} \geq \frac{g^2}{\sum_{tk}Cz_2^2(t)/K^2} = c\epsilon_3^2 \delta_0^2 K\frac{\zeta^4}{d^2}\sigmin^{\star2},
    \] since $\sum_{t}z_2^2(t) = \|z_2\|^2 = 1$.
    As a result, for a fixed $z_1\in S^{d-1}$ and $z_2 \in S^{T-1}$, by sub-Gaussian Hoeffding inequality, we have w.p. at least $1-\exp(-c\epsilon_3^2 \delta_0^2 K\frac{\zeta^4}{d^2}\sigmin^{\star2})$, 
    \[
        \langle\widehat{\Theta}_0-\mathbb{E}[\widehat{\Theta}_0], z_1z_2^\top \rangle \leqslant \epsilon_3 \delta_0\frac{\zeta^2}{d}\sigmin^{\star}.
    \]
    Applying epsilon net argument, let $\epsilon_3 = 0.1$, w.p. at least $1-\exp(T+d-c\epsilon_3^2 \delta_0^2 K\frac{\zeta^4}{d^2}\sigmin^{\star2})$, 
    \[
        \big\|\widehat{\Theta}_0-\mathbb{E}[\widehat{\Theta}_0]\big\| = \max_{\substack{z_1\in S^{d-1}\\ z_2 \in S^{T-1}}} \langle\hat{\Theta}_0-\mathbb{E}[\hat{\Theta}_0], z_1z_2^\top \rangle \leqslant 0.1\delta_0\frac{\zeta^2}{d}\sigmin^{\star}.
    \]
    This completes the proof.
\end{proof}
%
% \(K\geq C(T+d)/(\delta_0^2\frac{\zeta^4}{d^2}\sigma_{min}^{*2})\)
%
The theorem below bounds the SD between the estimated and the true latent representations of the reward matrices.
\begin{theorem}
\label{theorem:initial}
    Assume Assumption~\ref{low rank} holds. Pick a $\delta_0 \leqslant 0.1$. Then, with probability at least $1-\exp(T+d-c\epsilon_3^2 \delta_0^2 K\frac{\zeta^4}{d^2}\sigmin^{\star2})$, the estimated $\hat{B}$ from Algorithm~\ref{alg:whole} satisfies
    \[
        \text{SD}(\hB,\Bstar) \leqslant \delta_0.
    \]
\end{theorem}
\begin{proof}
    From Lemma~\ref{theta-Etheta}, we have $\big\|\widehat{\Theta}_0-\mathbb{E}[\widehat{\Theta}_0]\big\| \leqslant 0.1\delta_0\frac{\zeta^2}{d}\sigmin^{\star} \leqslant (1-1/\sqrt{2})(\sigma_r^\star(\mathbb{E}[\widehat{\Theta}_0]) - \sigma_{r+1}^\star(\mathbb{E}[\widehat{\Theta}_0]))$. 
    
We apply Proposition~\ref{prop:wedin}, with \(M = \widehat{\Theta}_0\) and \(M^{\star} = \mathbb{E}[\widehat{\Theta}_0] = \mathbb{E}[\psi\psi^\top]\Thetas\), and $\delta_0 \leqslant 0.1$. We have
    \[
        \text{SD}(\hB,\Bstar) \leqslant \frac{2\big\|\widehat{\Theta}_0-\mathbb{E}[\widehat{\Theta}_0]\big\|}{\frac{\zeta^2}{d}\sigmin^{\star}} \leqslant \frac{2 \cdot 0.1\delta_0\frac{\zeta^2}{d}\sigmin^{\star}}{\frac{\zeta^2}{d}\sigmin^{\star}} \leqslant \delta_0,
    \] w.p. at least \( 1-\exp[(T+d)-c K\frac{\zeta^4}{d^2}\sigmin^{\star2}]\).   
    % This is greater than \(1-\exp(-C(T+d))\) if \(K\geq C(T+d)/(\delta_0^2\frac{\zeta^4}{d^2}\sigma_{min}^{*2})\) for a large enough $C$.   
\end{proof}
Lemma below helps us to bound the estimation error of the low-rank reward matrices.
\begin{lemma}
\label{lemma:intermidiate_error}
    Consider the estimated $\widehat{B}$, $\widehat{W}$ and random vectors $\psi_{tk}$s from Algorithm~\ref{alg:whole}, with probability at least $1-2\exp(r-c\frac{\zeta^4}{d^2}K)$, 
    \begin{align*}
       & \Bigl\|
    \bigl( \hB^\top \Psi_t^\top \Psi_t \hB \bigr)^{-1}
    \hB^\top \Psi_t^\top \Psi_t (I - \hB \hB^\top) B^\star w_t^\star
\Bigr\|\\ &\leqslant 0.12\Bigl\| (I-\hB\hB^\top)\Bstar w_t^{\star} \Bigr\|.
    \end{align*}
\end{lemma}
\begin{proof}
    Consider a fixed $z \in S^{r-1}$, we have
    \[z^\top\hB^\top \Psi_t^\top \Psi_t \hB z = \sum^{K}_{k=1}z^\top\hB^\top\psi_{tk}\psi_{tk}^\top\hB z\]
    Let $X:=\psi_{tk}^\top\hB z$. From Proposition~\ref{prop:subg and others}, $X$ is a sub-Gaussian random variable with $\|X\|_{\psi_2} \leqslant C$. Then $X^2$ is sub-exponential random variable with $\|X^2\|_{\psi_1} \leqslant C\|X\|^2_{\psi_2} \leqslant C$. The summands are independent sub-exponential random variables with sub-exponential norm $\tau_k \leqslant C$.
    \begin{align}\label{eq:exp_Thetah}
    \mathbb{E}[z^\top\hB^\top\psi_{tk}\psi_{tk}^\top\hB z] = z^\top\hB^\top\mathbb{E}[\psi_{tk}\psi_{tk}^\top]\hB z \geq \frac{\zeta^2}{d}.
    \end{align}
    By sub-exponential Bernstein inequality, with $g = \epsilon_1K$, where $\epsilon_1 \leqslant 1$. Then, 
    \[
    \frac{g^2}{\sum^{K}_{k=1}\tau_k^2}\geq \frac{\epsilon_1^2K^2}{C^2K}=c\epsilon_1^2K
    \]
    \[
    \frac{g}{\max_k\tau_k} \geq \frac{\epsilon_1K}{C} = c\epsilon_1 K.
    \]
   For a fixed $z$, applying the Bernstein inequality and Eq.~\eqref{eq:exp_Thetah},
    \begin{align*}
          &\mathbb{P}\Bigl( \sum^{K}_{k=1}z^\top\hB^\top\psi_{tk}\psi_{tk}^\top\hB z-\frac{\zeta^2K}{d} \leqslant -\epsilon_1K \Bigr) \\& \leqslant \mathbb{P}\Bigl( \sum^{K}_{k=1}z^\top\hB^\top\psi_{tk}\psi_{tk}^\top\hB z - \mathbb{E}[\cdot] \leqslant -\epsilon_1K\Bigr)\leqslant \exp{\Bigl(-c\epsilon_1^2K\Bigr)},
    \end{align*}
    where $\mathbb{E}[\cdot]$ is the expected value of the first term.
    Using epsilon-net argument, by setting $\epsilon_1=0.1\frac{\zeta^2}{d}$, we can conclude that with probability $1-\exp(r-c\frac{\zeta^4}{d^2}K)$, $\min_{z\in S^{d-1}}z^\top\hB^\top \Psi_t^\top \Psi_t \hB z \geq 0.9\frac{\zeta^2}{d}K$.
    \begin{align}
   \bigl\| \bigl( \hB^\top \Psi_t^\top \Psi_t \hB \bigr)^{-1} \bigr\|
    &= \hspace{-1 mm}\frac{1}{\sigma_{\min}\!\left( \hB^\top \Psi_t^\top \Psi_t \hB \right)}\le \frac{1}{0.9\frac{\zeta^2}{d}K}.\label{eq:inv_term}
    \end{align}
    Consider $z \in S^{r-1}$, we have $    z^\top\hB^\top\Psi_t^\top\Psi_t(I-\hB\hB^\top)\Bstar w_t^{\star} $
    \[
= \sum^{K}_{k=1}\bigl(\psi_{tk}^\top\hB z\bigr)^\top\bigl(\psi_{tk}^\top(I-\hB\hB^\top)\Bstar w_t^{\star}\bigr).
    \]
    By Proposition~\ref{prop:subg and others}, we have 
    \begin{align*}
        \|\psi^\top\hB z\|_{\psi_2} &\leqslant \|\hB z\|\cdot\|\psi\|_{\psi_2} \leqslant C\text{~and}\\
        \|\psi^\top(I-\hB\hB^\top)\Bstar w_t^{\star}\|_{\psi_2} &\leqslant C\|(I-\hB\hB^\top)\Bstar w_t^{\star}\|.
    \end{align*}
    As a result, the summands are independent sub-exponential random variables with sub-exponential norm $\tau_k \leqslant \|\psi^\top\hB z\|_{\psi_2}\|\psi^\top(I-\hB\hB^\top)\Bstar w_t^{\star}\|_{\psi_2} \leqslant C\|(I-\hB\hB^\top)\Bstar w_t^{\star}\|$.

 Applying sub-exponential Bernstein inequality with $g = \epsilon_2K\|(I-\hB\hB^\top)\Bstar w_t^{\star}\|$, we have
    \begin{align*}
         &\mathbb{E}[\sum^{K}_{k=1}\bigl(\psi_{tk}^\top\hB z\bigr)^\top\bigl(\psi_{tk}^\top(I-\hB\hB^\top)\Bstar w_t^{\star}] \\
         &= \sum^{K}_{k=1}z^\top\hB^\top\mathbb{E}[\psi_{tk}\psi_{tk}^\top](I-\hB\hB^\top)\Bstar w_t^{\star} \\
         &\leqslant \frac{1}{d\xi^2}\sum^{K}_{k=1}z^\top\hB^\top(I-\hB\hB^\top)\Bstar w_t^{\star}= 0.
    \end{align*}
    Further,
    \[
\frac{g^2}{\sum^{K}_{k=1}\tau_k^2} \geq \frac{\epsilon_2^2K^2\|(I-\hB\hB^\top)\Bstar w_t^{\star}\|^2}{KC^2\|(I-\hB\hB^\top)\Bstar w_t^{\star}\|^2} = c\epsilon_2^2K
    \]
    \[
        \frac{g}{\max_k\tau_k} \geq \frac{\epsilon_2K\|(I-\hB\hB^\top)\Bstar w_t^{\star}\|}{C\|(I-\hB\hB^\top)\Bstar w_t^{\star}\|} = c\epsilon_2K.
    \]
   For a fixed $z \in S^{r-1}$, we have
    \[
    \mathbb{P}\Bigl(z^\top\hB^\top\Psi_t^\top\Psi_t(I-\hB\hB^\top)\Bstar w_t^{\star} \geq g\Bigr) \leqslant \exp{\Bigl(-c\epsilon_2^2K\Bigr)}.
    \]
    Using epsilon-net argument, setting $\epsilon_2 = 0.1\frac{\zeta^2}{d}$, with probability at least $1-\exp{(r-c\frac{\zeta^4}{d^2}K)}$, we have
    \begin{align}
    &\bigl\| \hB^\top \Psi_t^\top \Psi_t (I - \hB \hB^\top)\Bstar w_t^\star \bigr\| \\
    &= \max_{z \in S^{d-1}}
        z^\top\hB^\top\Psi_t^\top\Psi_t(I-\hB\hB^\top)\Bstar w_t^{\star}\nn \\
    &\le 0.1\frac{\zeta^2}{d}K
        \bigl\| (I-\hB\hB^\top)\Bstar w_t^{\star} \bigr\|.\label{eq:I_term}
    \end{align}
    Combining Eqs.~\eqref{eq:inv_term} and~\eqref{eq:I_term}, we have with probability at least $1-2\exp(r-c\frac{\zeta^4}{d^2}K)$
    \begin{align*}
        &\bigl\|
    \bigl( \hB^\top \Psi_t^\top \Psi_t \hB \bigr)^{-1}
    \hB^\top \Psi_t^\top \Psi_t (I - \hB \hB^\top) B^\star w_t^\star
\bigr\| \\&\leqslant \bigl\|\bigl(\hB^\top \Psi_t^\top \Psi_t \hB \bigr)^{-1}\bigr\| \cdot \bigl\|\hB^\top \Psi_t^\top \Psi_t (I - \hB \hB^\top) B^\star w_t^\star\bigr\| \\
    &\le \frac{1}{\bigl(\frac{\zeta^2}{d}-\epsilon_1\big)K} \times \epsilon_2K
        \bigl\| (I-\hB\hB^\top)\Bstar w_t^{\star} \bigr\| \\
    &= \frac{0.1\frac{\zeta^2}{d}}{\frac{\zeta^2}{d}-0.1\frac{\zeta^2}{d}}\bigl\| (I-\hB\hB^\top)\Bstar w_t^{\star} \bigr\| \\
    &\leqslant 0.12 \bigl\| (I-\hB\hB^\top)\Bstar w_t^{\star} \bigr\|.
    \end{align*}
    This completes the proof.
\end{proof}
We are now ready to bound the estimation error.
\begin{theorem}
\label{theorem:columnwise_error}
    Assume Assumption \ref{low rank} holds and consider the estimated $\widehat{B}$, $\widehat{W}$ and random vectors $\psi_{tk}$s from Algorithm~\ref{alg:whole}. If the number of samples $K \geq C\frac{r(T+d)d^2}{\delta_0^2\sigmin^{*2}\zeta^4}$, then with probability at least $1-3d^{-10}$,
    \[
        \|\hB\hat{w_t}-\Bstar \wstar_t\| \leqslant 1.12\delta_0\sqrt{d}.
    \]
\end{theorem}
\begin{proof}
Based on least square estimation of $\hat{w}_t$
\begin{align*}
\hat{w}_t
&= \big( \hB^\top \Psi_t^\top \Psi_t \hB \big)^{-1}
    (\Psi_t \hB)^\top Y_t
    \end{align*}
    \begin{align*}
&= \big( \hB^\top \Psi_t^\top \Psi_t \hB \big)^{-1}
    (\Psi_t \hB)^\top
    \big( \Psi_t B^\star w_t^\star \big) \\
&= \big( \hB^\top \Psi_t^\top \Psi_t \hB \big)^{-1}
    \hB^\top \Psi_t^\top \Psi_t \hB \hB^\top B^\star w_t^\star\\
    &+ \big( \hB^\top \Psi_t^\top \Psi_t \hB \big)^{-1}
    \hB^\top \Psi_t^\top \Psi_t
    (I - \hB \hB^\top) B^\star w_t^\star 
            \end{align*}
    \begin{align*}
&= \hB^\top B^\star w_t^\star \hspace{-1 mm}+\hspace{-1 mm} \big( \hB^\top \Psi_t^\top \Psi_t \hB \big)^{-1}
    \hB^\top \Psi_t^\top \Psi_t
    (I - \hB \hB^\top) B^\star w_t^\star .
\end{align*}
Applying $\hB$ to both sides, we have $\hB \hat{w}_t$
\begin{align*}
&= \hB \hB^\top B^\star w_t^\star \\ &+ \hB \big( \hB^\top \Psi_t^\top \Psi_t \hB \big)^{-1}\hB^\top \Psi_t^\top \Psi_t\big(I - \hB \hB^\top\big) B^\star w_t^\star \\
&= B^\star w_t^\star + \big(\hB\hB^\top-I\big)B^\star w_t^\star\\ &+ \hB \big( \hB^\top \Psi_t^\top \Psi_t \hB \big)^{-1}\hB^\top \Psi_t^\top \Psi_t\big(I - \hB \hB^\top\big) B^\star w_t^\star
\end{align*}
As a result,
\begin{align*}
&\bigl\| \hB \hat{w}_t - B^\star w_t^\star \bigr\|
= \bigl\|(\hB \hB^\top - I) B^\star w_t^\star \\& + \hB \bigl( \hB^\top \Psi_t^\top \Psi_t \hB \bigr)^{-1}
    \hB^\top \Psi_t^\top \Psi_t
    (I - \hB \hB^\top) B^\star w_t^\star \bigr\| \\
&\le
\bigl\| (\hB \hB^\top - I) B^\star w_t^\star \bigr\| \\&+ \bigl\|
    \hB \bigl( \hB^\top \Psi_t^\top \Psi_t \hB \bigr)^{-1}
    \hB^\top \Psi_t^\top \Psi_t
    (I - \hB \hB^\top) B^\star w_t^\star
\bigr\| \\
&\le
\bigl\| (\hB \hB^\top - I) B^\star w_t^\star \bigr\| \\& + \bigl\|
    \bigl( \hB^\top \Psi_t^\top \Psi_t \hB \bigr)^{-1}
    \hB^\top \Psi_t^\top \Psi_t (I - \hB \hB^\top) B^\star w_t^\star\bigr\| \\
&\le \bigl\| (\hB \hB^\top - I) B^\star w_t^\star \bigr\| + 0.12\bigl\| (\hB \hB^\top - I) B^\star w_t^\star \bigr\|\\
& \le 1.12\delta_0\bigl\|\thetas_t\bigr\| \leqslant 1.12\delta_0\sqrt{d}
\end{align*}
with probability $1-2\exp(r-c\frac{\zeta^4}{d^2}K)-\exp(T+d-c\delta_0^2 K\frac{\zeta^4}{d^2}\sigmin^{\star2})$. To ensure that the probability is at least $1-3d^{-10}$, each exponential term must be smaller than or equal to $d^{-10}$. We have,
\[
    r-c\frac{\zeta^4}{d^2}K \leqslant -10\log(d) \Rightarrow K \geq C\frac{d^2r}{\zeta^4}
\]
\[
    T+d-c\delta_0^2 K\frac{\zeta^4}{d^2}\sigmin^{\star2} \leqslant -10\log(d) \Rightarrow K \geq C\frac{(T+d)d^2}{\delta_0^2\sigmin^{\star2}\zeta^4}
\]
Combining these results, we have $K \geq C\frac{r(T+d)d^2}{\delta_0^2\sigmin^{*2}\zeta^4}$.
\end{proof}
\begin{remark}
    Theorem~\ref{theorem:columnwise_error} reflects a fundamental trade-off between the number of samples required and the estimation error. As the desired subspace distance $\delta_0$ decreases, the column-wise estimation error also decreases, and the required number of samples will increase. 

    Since the feature vectors $\psi_{tk}$s are generated from the environment through the specified data collection process, rather than being freely designed, their distribution and properties directly impact the sample complexity. The weaker spectrum (the singular value of $\mathbb{E}[\psi\psi^\top]$ is scaled by $1/d$) introduces a $d^2$ increase in sample complexity. A constant upper bound for $\|\psi_{tk}\|$ also affects the sample complexity. 
    %{\cblue These are the two main reasons for difference in sample complexity compared to settings where the distribution of $\psi_{tk}$ can be more controlled.}
\end{remark}
\subsection{Regret Analysis}
We first present a lemma to bound the per-step regret and value error and then in Theorem~\ref{theorem:regret_bound} we present the cumulative regret bound of our algorithm.
\begin{lemma}
\label{lemma:regret_induction}
Let $\hat{V}^\Pi_t(s_h^n)$ be the value of task $t$ for a state at horizon $h$ and episode $n$ under policy $\Pi$ calculated from estimated $\hat{\R}_t$. Let \(\delta_{t,h,1}^n := V^{\star}_t(s_h^n) -\hat{V}^{\hat{\Pi}^{\star}(t)}_t(s_h^n)\) and \(\delta_{t,h,2}^n := \hat{V}^{\hat{\Pi}^{\star}(t)}_t(s_h^n) - V_t^{\hat{\Pi}^{\star}(t)}(s_h^n)\), $\text{Err} := \|\hB \hat{w}_t - B^\star w_t^\star \|$. Then,
    \[
        \delta_{t,h,1}^n \leqslant \text{Err}\cdot (H-h+1)\text{~and~}
        \delta_{t,h,2}^n \leqslant \text{Err}\cdot(H-h+1).
    \]
\end{lemma}
\begin{proof}
    We prove this lemma by induction. 
    For the base case, $h = H$, we need to prove that $\delta_{t,H,1}^n \leqslant \text{Err}$ and $\delta_{t,H,2}^n \leqslant \text{Err}$. Since \(V^{\star}_t(s_{H+1}^n) = 0, \hat{V}^{\hat{\Pi}^{\star}(t)}_t(s_{H+1}^n)=0\), let $a= \arg \max_a R(s_H^n,a)$ and $\hat{a}= \arg \max_{\hat{a}} \hat{R}(s_H^n,\hat{a})$,
    \begin{align*}
        \delta_{t,H,1}^n &= \max_a R(s_H^n,a) - \max_{\hat{a}} \hat{R}(s^n_H,\hat{a}) \\
        &\leqslant R(s_H^n,a) - \hat{R}(s^n_H,a)\\
        &= \theta_t^\top \psi(s_H^n,a) - \hat{\theta}_t^\top\psi(s_H^n,a) \\
        &= (\theta_t-\hat{\theta}_t)^\top \psi(s_H^n,a) \\
        &\leqslant \|\theta_t-\hat{\theta}_t\| \cdot\|\psi(s_H^n,a)\| 
        \leqslant \text{Err}.
    \end{align*}
    Similarly, for $\delta_{t,H,2}^n$, we have
    \begin{align*}
        \delta_{t,H,2}^n &= \hat{R}(s_H^n,\hat{a}) - R(s^n_H,\hat{a}) \\
        &= \hat{\theta}_t^\top\psi(s_H^n,\hat{a})-\theta_t^\top \psi(s_H^n,\hat{a})  
        %&= (\hat{\theta}_t-\theta_t)^\top \psi(s_H^n,\hat{a}) \\
        %&\leqslant \|\hat{\theta}_t-\theta_t\| \cdot\|\psi(s_H^n,\hat{a})\| \\
        \leqslant \text{Err}.
    \end{align*}
    By the induction assumption, we assume that $\delta_{t,h+1,1}^n \leqslant \text{Err}(H-h)$ and $\delta_{t,h+1,2}^n \leqslant \text{Err}(H-h)$. To prove the induction hypotheses, we have
    \begin{align*}
        \delta_{t,h,1}^n &= V^{\star}_t(s_h^n) -\hat{V}^{\hat{\Pi}^{\star}(t)}_t(s_h^n) \\
        & = \scalemath{0.85}{Q^{\star}_t(s_h^n,a_h^n) -\hat{Q}^{\hat{\Pi}^{\star}(t)}_t(s_h^n,\hat{a}_h^n)
        \leqslant Q^{\star}_t(s_h^n,a_h^n) -\hat{Q}^{\hat{\Pi}^{\star}(t)}_t(s_h^n,a_h^n)} \\
        &= R(s_h^n,a_h^n) + \sum_{s_{h+1}^n \in \S}\mathbb{P}(s_{h+1}^n|s_h^n,a_h^n)V_t^{\star}(s_{h+1}^n) \\&- \hat{R}(s_h^n,a_h^n) - \sum_{s_{h+1}^n \in \S}\mathbb{P}(s_{h+1}^n|s_h^n,a_h^n)\hat{V}_t^{\hat{\Pi}^{\star}(t)}(s_{h+1}^n) \\
        &= \sum_{s_{h+1}^n \in \S}\mathbb{P}(s_{h+1}^n|s_h^n,a_h^n)\Bigl(V_t^{\star}(s_{h+1}^n)-\hat{V}_t^{\hat{\Pi}^{\star}(t)}(s_{h+1}^n)\Bigl) \\&+ (R(s_h^n,a_h^n)-\hat{R}(s_h^n,a_h^n)) \\
        & \leqslant \text{Err}(H-h) + \text{Err} = \text{Err}(H-h+1)
    \end{align*}
    Following the same steps, we can show
    $$\delta_{t,h,2}^n \leqslant \text{Err}(H-h+1).$$
    This completes the proof.
    % Similarly,
    % \begin{align*}
    %     \delta_{t,h,2}^n &= \hat{V}^{\hat{\Pi}^{\star}(t)}_t(s_h^n) - V_t^{\hat{\Pi}^{\star}(t)}(s_h^n) \\
    %     & = \hat{Q}^{\hat{\Pi}^{\star}(t)}_t(s_h^n,\hat{a}_h^n) - Q^{\hat{\Pi}^{\star}(t)}_t(s_h^n,\hat{a}_h^n) \\
    %     &= \hat{r}(s_1^n,\hat{a}_h^n) + \sum_{\forall s_{h+1}^n}\mathbb{P}(s_{h+1}^n|s_h^n,\hat{a}_h^n)\hat{V}_t^{\hat{\Pi}^{\star}(t)}(s_{h+1}^n) \\&- r(s_h^n,\hat{a}_h^n) - \sum_{\forall s_{h+1}^n}\mathbb{P}(s_{h+1}^n|s_h^n,\hat{a}_h^n)V_t^{\hat{\Pi}^{\star}(t)}(s_{h+1}^n) \\
    %     &= \sum_{\forall s_{h+1}^n}\mathbb{P}(s_{h+1}^n|s_h^n,\hat{a}_h^n)\Bigr(\hat{V}_t^{\hat{\Pi}^{\star}(t)}(s_{h+1}^n)-V_t^{\hat{\Pi}^{\star}(t)}(s_{h+1}^n)\Bigl) \\&+ (\hat{r}-r)(s_h^n,\hat{a}_h^n) \\
    %     & \leq \text{Err}(H-h) + \text{error} \\
    %     &= \text{Err}(H-h+1).
    % \end{align*}   
\end{proof}
\begin{figure*}[t!]
    \centering
    % First subfigure
    \begin{subfigure}[b]{0.97\textwidth}
        \centering
        \includegraphics[width=\textwidth, height = 3.2cm]{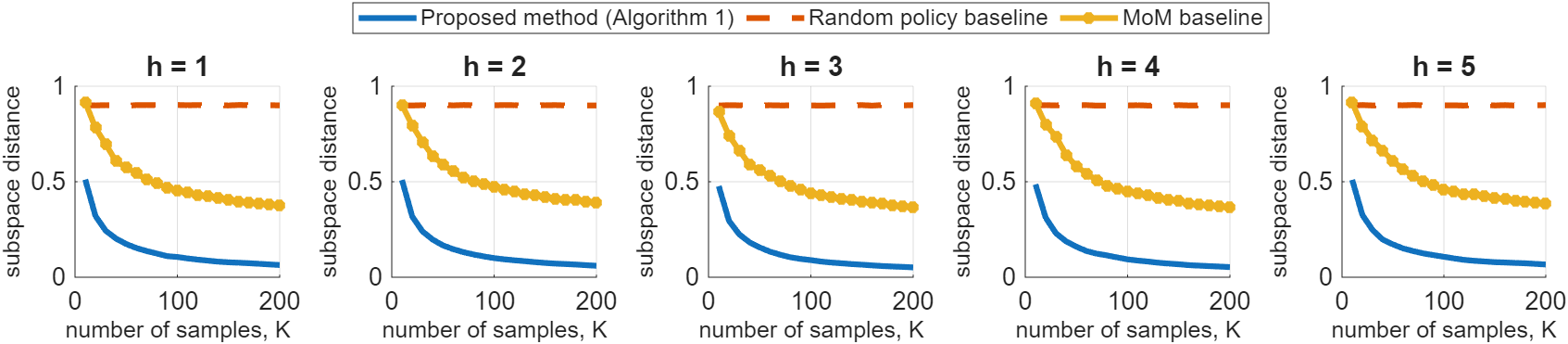}
        \caption{Subspace distance as a function of number of samples $K$. Subspace distance is defined as $\SD(\hB_h, B_h^\star)$.}
        \label{fig:1a}
    \end{subfigure}
    %\hfill
    \begin{subfigure}[b]{0.97\textwidth}
        \centering
        \includegraphics[width=\textwidth, height = 3.3cm]{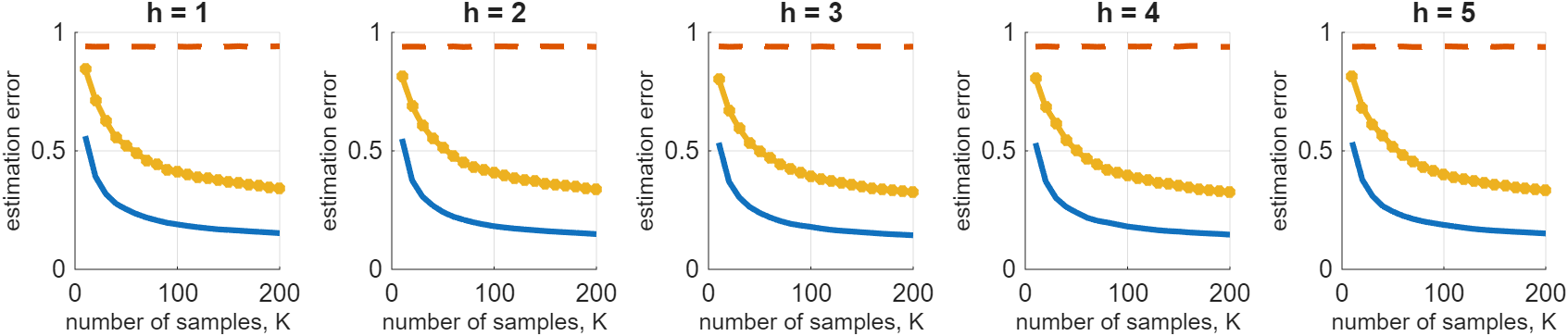}%\vspace{-1 mm}
        \caption{Estimation error as a function of the number of samples $K$. Estimation error is defined as $\|\widehat{\Theta}_h-\Theta^\star_h\|_\text{F}/\|\Theta^\star_h\|_\text{F}$}
        \label{fig:1b}
    \end{subfigure}
    \caption{{\em Results for Experiment~1:} We set $d=100, T=100,r=2,|\S|=1000,|\A|=10.$ Random policy baseline replaces stage~2 in Algorithm~\ref{alg:whole} with a random policy that chooses each action uniformly. MoM baseline replaces stage~3 in Algorithm~\ref{alg:whole} with an MoM estimator. Our proposed approach outperforms both the baselines.}
    % 'd=100, T=100, r=2, S=1000, A=10'
    \label{fig:error}
\end{figure*}
\begin{figure*}[t!]
\vspace{-2 mm}
    \begin{subfigure}[t]{0.32\textwidth}
        \centering
        \includegraphics[width=\textwidth,height=3.9cm]{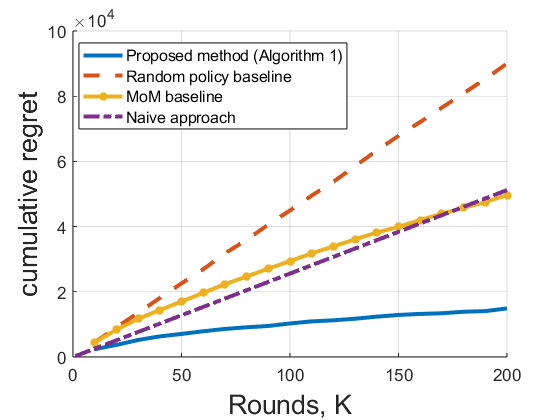}%\vspace{-1 mm}
        \caption{}
        \label{fig:2a}
    \end{subfigure}
    \hfill
    \begin{subfigure}[t]{0.32\textwidth}
        \centering
        \includegraphics[width=\textwidth,height=3.9cm]{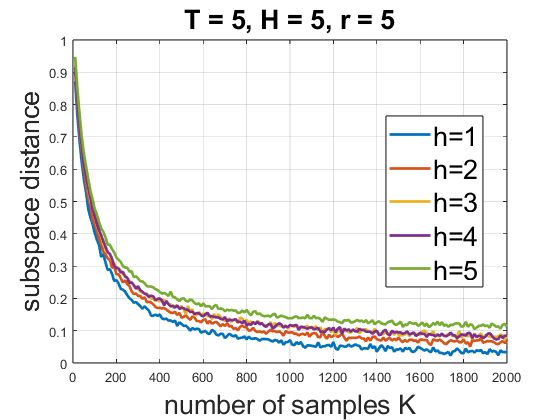}%\vspace{-1 mm}
        \caption{}
        \label{fig:2b}
    \end{subfigure}
    \hfill
    \begin{subfigure}[t]{0.32\textwidth}
        \centering
        \includegraphics[width=\textwidth,height=3.9cm]{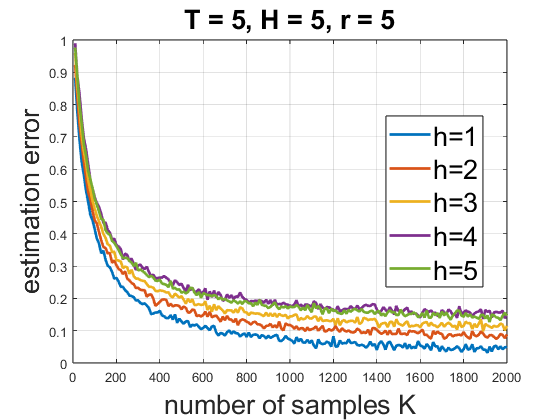}%\vspace{-1 mm}
        \caption{}
        \label{fig:2c}
    \end{subfigure}
    \vspace{-2 mm}
    \caption{{\em Results for Experiment~1:} Figure~\ref{fig:2a} presents regret plots for Experiment~1 of the proposed algorithm against the three baseline approaches. {\em Results for Experiment~2:}
     Figure~\ref{fig:2b} and~\ref{fig:2c} present results for the grid maze environment. The five goal points for five tasks are (1,1), (2,2), (3,3), (4,4), (5,5), respectively.  Figure~\ref{fig:2b} presents subspace distance vs. number of samples $K$. Fig.~\ref{fig:2c} presents the estimation error vs. number of samples $K$. Estimation error is defined as $\|\widehat{\Theta}_h-\Theta^\star_h\|_\text{F}/\|\Theta^\star_h\|_\text{F}$.}
    \label{fig:maze}
\end{figure*}
\begin{theorem}\label{theorem:regret_bound}
The regret of the MTRL-RL algorithm in Algorithm~\ref{alg:whole} for $T$ tasks, and $N$ rounds is given by
\begin{align*}
\text{Reg}(N,T) &= \sum_{n=1}^{N} \sum_{t=1}^{T} V^{\star}_t(s_1^n) - V_t^{\hat{\Pi}^{\star}(t)}(s_1^n) = O(NTH\sqrt{d}\delta_0).
\end{align*}
\end{theorem}
\begin{proof}
Using the definition of regret,
    \begin{align*}
        &\text{Reg}(N,T) = \sum_{n=1}^{N} \sum_{t=1}^{T} \Big(V^{\star}_t(s_1^n) - V_t^{\hat{\Pi}^{\star}(t)}(s_1^n)\Big) \\
        &= \scalemath{0.95}{\sum_{n=1}^{N} \sum_{t=1}^{T}\Big( V^{\star}_t(s_1^n) -\hat{V}^{\hat{\Pi}^{\star}(t)}_t(s_1^n)+\hat{V}^{\hat{\Pi}^{\star}(t)}_t(s_1^n) - V_t^{\hat{\Pi}^{\star}(t)}(s_1^n) \Big)} 
        \end{align*}
        \begin{align*}
        &\leqslant \sum_{n=1}^{N} \sum_{t=1}^{T}\Big(\delta_{t,1,1}^n +\delta_{t,1,2}^n\Big) \leqslant 2NTH\cdot \text{Err} \leqslant 2.5NTH\sqrt{d}\delta_0.
    \end{align*}
    This completes the proof.
\end{proof}
\begin{remark}
   Theorem~\ref{theorem:regret_bound} shows that the cumulative regret is controlled by the estimation error. In particular, if the error term is zero, i.e., $\delta_0$, then the bound implies $\text{Regret}(N,T) = 0$, and the learned policy $\hat{\Pi}^*$ is optimal. In practice, however, due to statistical uncertainty from finite samples, zero regret is not available. The regret bound implies that the learned policy is $\epsilon$-optimal with \(\epsilon = \mathcal{O}(H\sqrt{d}\delta_0)\).
\end{remark}
\section{Numerical Analysis}\label{sec: Simulations}
In this section, we present numerical experiments to verify the performance of our proposed algorithm. We evaluate our method on a simulated multi-task control environment and a grid maze navigation problem. 
To compare our approach against baselines, we use three metrics. The subspace distance of the latent representation, $\SD(\hB_h, B_h^\star)$, the estimation error of the reward parameter matrix, $\|\widehat{\Theta}_h-\Theta^\star_h\|_\text{F}/\|\Theta^\star_h\|_\text{F}$, and the regret of the learned policy. To compute the regret we evaluate the optimal policy and the value vector using the true reward model $R_{ht}(s,a) = \langle\thetas_{ht},\psi(s,a)\rangle$ and the estimated model $\hat{R}_{ht}(s,a) = \langle\hat{\theta}_{ht},\psi(s,a)\rangle$.
% {\cblue We use value iteration to calculate the optimal policy for any given $\R$.} 
All results are averaged over 100 independent trials. 
% In the third experiment, we alter the parameters $d, T ,$ and $|\S|$. All results are averaged over 100 independent trials. 
%
%\subsection{Results and Analysis}
\subsubsection{Experiment 1}
We evaluated our method on a simulated multi-task control environment. We set $d=100, T=100,r=2,|\S|=1000,|\A|=10$. The $B^\star_h$s, for $h\in[H]$, is generated by orthonormalizing i.i.d. standard Gaussian matrix. The $W^\star_h$s are generated from i.i.d. Gaussian distribution. For all $(s,a)$ pairs, $\phi(s,a)$ are generated from i.i.d. Gaussian distribution and normalized by $\ell_1$ norm. For each state $s \in \S$, we randomly select half of the actions $a\in \A$ and their embedding $\psi(s,a)$ is generated from a Gaussian distribution with zeros mean and covariance $\frac{1}{\sqrt{d}}I$. For the remaining $(s,a)$ pair, we set their embedding $\psi(s,a)$ to be $e_1$. The intuition is to mimic that in many real-life RL environments, some $\psi(s,a)$s are nicely formulated while some other $\psi(s,a)$s are poorly constructed ($e_1$ in this case).

Recall that Algorithm~\ref{alg:whole} contains four stages. We considered three baselines: (i) a random policy baseline that replaces stage~2 with a random policy that chooses each action uniformly, rather than employing a specifically designed exploration policy, (ii) an MoM baseline that replaces stage~3 with an MoM estimator, and (iii) a naive approach baseline that independently solves $T$ tasks by Thompson sampling (TS). The random policy baseline is designed to illustrate the importance of designing the exploration policy $\hat{\Pi}$ in Algorithm~\ref{alg:whole} (stage~2). The MoM estimator serves as a baseline for our representation learning approach. It uses exploration policy $\hat{\Pi}$ from Algorithm~\ref{alg:whole} and estimates $\widehat{B}$ by calculating the top-$r$ singular vector of $\widehat{\Theta} = \frac{1}{KT}\sum_{k=1,t=1}^{K,T}y_{tk}^2\psi_{tk}\psi_{tk}^\top$ \cite{tripuraneni2021provable}. The naive approach serves as a baseline for solving the tasks independently rather than jointly. 

From Figures \ref{fig:error} and \ref{fig:2a}, our proposed method consistently outperforms the three benchmarks. Using a uniform policy leads to poor estimation performance. The underlying reason is that random sampling collects degenerate features $\psi(s,a) = e_1$, causing the empirical covariance to be nearly singular, illustrating the importance of finding an appropriate exploration policy $\hat{\Pi}$.
The key advantage of finding the exploration policy $\hat{\Pi}$ lies in its ability to generate informative feature distributions, which is essential for accurate low-rank matrix recovery. This highlights that representation learning in RL is fundamentally limited by the quality of collected data, rather than the estimation procedure alone.

\subsubsection{Experiment 2}
% \begin{figure}[t!]
%     \centering
%     % First subfigure
%     \begin{subfigure}[b]{0.45\textwidth}
%         \centering
%         \includegraphics[width=\textwidth,height=4cm]{SD_maze.png}
%         \caption{Subspace distance as a function of number of samples $K$}
%         \label{fig:3a}
%     \end{subfigure}
%     \hfill
%     \begin{subfigure}[b]{0.45\textwidth}
%         \centering
%         \includegraphics[width=\textwidth]{CS_maze.png}
%         \caption{Esimation error as a function of number of samples $K$. Estimation error is defined as $\|\widehat{\Theta}_h-\Theta^\star_h\|_\text{F}/\|\Theta^\star_h\|_\text{F}$}
%         \label{fig:3b}
%     \end{subfigure}
%     \caption{The five goal points for five tasks are (1,1), (2,2), (3,3), (4,4), (5,5) respectively.}
%     \label{fig:maze}
% \end{figure}
We evaluated our algorithm on a 5 $\times$ 5 grid maze navigation problem. The state space consists of $|\S| = 25$ grid cells, and the action space includes four moves (up, down, left, right). 
For each $(s,a)$ pair, we use the canonical features $\psi(s,a)=e_{i(s,a)}\in \mathbb{R}^d$ with $d=|\S|\times|\A|=100$, where $i(s,a)$ indexes the pair.
%We use canonical features $\psi(s,a)=e_{(s,a)}\in \mathbb{R}^d$ with $d=|\S|\times|\A|=100$. 
We considered $T=5$ tasks. Each task corresponds to a specific goal location, and the reward is based on the Manhattan distance to the goal point. We vary the sample sizes from $K=10$ to $K=2000$ in increments of 10. From Figures~\ref{fig:2b} and~\ref{fig:2c}, we observe that our proposed method succeed to learn an accurate representation in the grid environment.
\vspace{-2 mm}
\section{Conclusion and Future Work}\label{sec:Conclusion}
This paper investigated multi-task representation learning for reinforcement learning under a shared low-rank structure. 
We considered a linear reward model and proposed a computationally efficient approach that uses reward free RL for essential policy construction to guide data collection, combined with reward parameter estimation to jointly recover the reward matrix and learn an $\epsilon$-optimal policy.
%We proposed a computationally efficient approach that combines spectral initialization with task-wise least-squares estimation to jointly recover a common representation and task-specific parameters under general feature distributions. 
Our analysis showed that accurate low-rank recovery is achievable without relying on restrictive assumptions such as incoherence or Gaussian features, and we quantified the associated sample complexity. Numerical experiments validated the effectiveness of the proposed method, demonstrating improved representation accuracy and performance over baseline approaches on two MDP environments.
As part of future work, we will extend the proposed framework to widely used RL algorithms, like PPO and DQN.
% Some potential future works include decentralized setting.
%
%\vspace{-2 mm}
\bibliographystyle{IEEEtran}
\bibliography{ref.bib, bandits.bib}
\end{document}